\documentclass{article}

\usepackage{microtype}
\usepackage{graphicx}
\usepackage{subfigure}
\usepackage{amsfonts}
\usepackage{enumitem}
\usepackage{mathtools}
\usepackage{amsmath,amssymb,amsthm,relsize}
\usepackage{booktabs} 

\usepackage{hyperref}

\usepackage[accepted]{icml/icml2020}

\newtheorem{theorem}{Theorem}
\newtheorem{lemma}{Lemma}
\newtheorem{proposition}{Proposition}

\newcommand{\allfor}{\displaystyle\mathop{\mathlarger{\forall}}}
\newcommand{\xeist}{\displaystyle\mathop{\mathlarger{\exists}}}

\icmltitlerunning{A Computationally Efficient Neural Network Invariant to the Action of Symmetry Subgroups}

\begin{document}

\twocolumn[
\icmltitle{A Computationally Efficient Neural Network Invariant to the Action \\ of Symmetry Subgroups}



\icmlsetsymbol{equal}{*}

\begin{icmlauthorlist}
\icmlauthor{Piotr Kicki}{irmi}
\icmlauthor{Mete Ozay}{}
\icmlauthor{Piotr Skrzypczyński}{irmi}
\end{icmlauthorlist}

\icmlaffiliation{irmi}{Institute of Robotics and Machine Intelligence, Poznan University of Technology, Poznan, Poland}

\icmlcorrespondingauthor{Piotr Kicki}{piotr.z.kicki@doctorate.put.poznan.pl}

\icmlkeywords{Machine Learning, deep neural networks, group invariance, G-invariance, geometric deep learning, ICML}

\vskip 0.3in
]



\printAffiliationsAndNotice{} 
\begin{abstract}
We introduce a method to design a computationally efficient $G$-invariant
neural network that approximates functions invariant to the action of a given
permutation subgroup $G \leq S_n$ of the symmetric group on input data.
The key element of the proposed network architecture is a new $G$-invariant transformation module, which produces a $G$-invariant latent representation of the input data.
This latent representation is then processed with a multi-layer perceptron in the network.
We prove the universality of the proposed architecture, discuss its properties and highlight its computational and memory efficiency.
Theoretical considerations are supported by numerical experiments involving different network configurations,
which demonstrate the effectiveness and strong generalization properties of the proposed method in comparison to other $G$-invariant neural networks.

\end{abstract}

\section{Introduction}

The design of probabilistic models which reflect symmetries existing in data is considered an important task following the notable success of deep neural networks, such as convolutional neural networks (CNNs) \citep{cnn2} and PointNet \citep{pointnet}.
Using prior knowledge about the data and expected properties of the model, such as permutation invariance \citep{pointnet}, one can propose models that achieve superior performance. Similarly, translation equivariance can be exploited for CNNs \citep{G_equiv_cnns} to reduce their number of weights.


Nevertheless, researchers have been working on developing a general approach which enables to design architectures that are invariant and equivariant to the action of particular groups $G$. Invariance and equivariance of learning models to actions of various groups $G$ are discussed in the literature \citep{deep_sets, cohen2, parameter_sharing}. 
However, in this paper, we only consider invariance to permutation groups $G$, which are the subgroups\footnote{A subset $G \subset S_n$ is a subgroup of $S_n$ if and only if it satisfies group properties. Please see the appendix \ref{sec:math} for the formal definitions.} of the symmetric group $S_n$ of all permutations on a finite set of $n$ elements, as it covers many interesting applications.

\begin{figure}[t]
\begin{center}
\centerline{\includegraphics[width=\columnwidth]{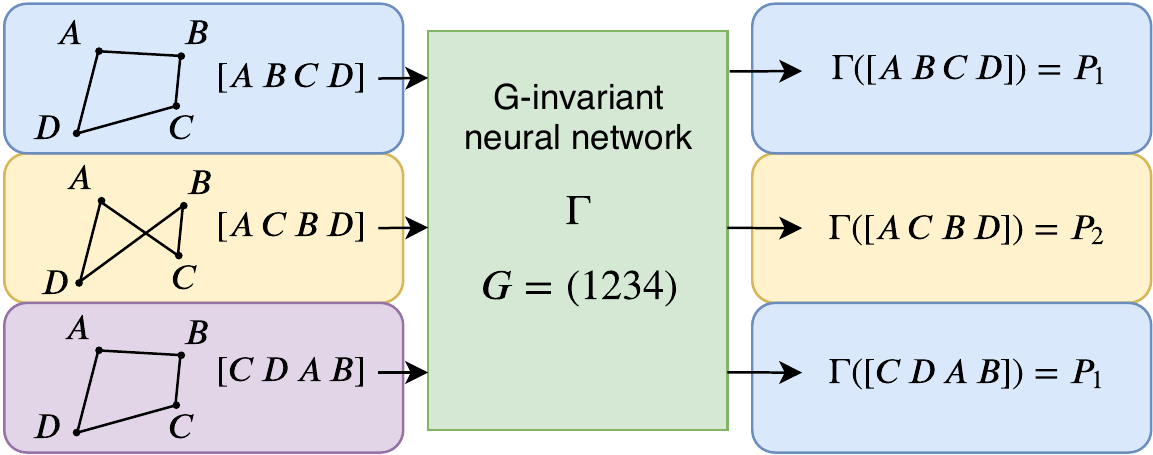}}
\caption{An illustration of employment of the proposed {$G$-invariant} neural network $\Gamma$ for estimation of area of quadrangles. No matter which vertex of a quadrangle $[A\,B\,C\,D]$ is given first, if the consecutive vertices are provided in the same order (e.g. for the quadrangle $[C\,D\,B\,A]$), then the network $\Gamma$ computes the same area $P_1$. However, the network $\Gamma$ is not invariant to all permutations. For example, the shape of $[A\,C\,B\,D]$ is hourglass-like, and the order of its vertices is different from that of the other quadrangles. Therefore, the network $\Gamma$ estimates a different area $P_2$ (please see the examples in the yellow boxes).}
\label{fig:intro}
\end{center}
\vskip -0.4in
\end{figure}

An example of the employment of the proposed $G$-invariant network for a set of quadrangles is illustrated in Figure~\ref{fig:intro}. The network $\Gamma$ receives a matrix representation of the quadrangles (i.e. a vector of 4 points on a plane) and outputs the areas covered by those quadrangles. One can spot that, no matter which point will be given first, if the consecutive vertexes are provided in the right order, the area of the figure will remain the same. Such property can be described as $G$-invariance, where $G = (1234)$\footnote{$G=(1234)$ denotes a group $G$ generated by the permutation $(1234)$, in which the first element is replaced by the second, second by the third and so on, till last element being replaced by first.}.

Recently, \citet{maron19} proposed a $G$-invariant neural network architecture for some finite subgroups $G \leq S_n$ and proved its universality. Unfortunately, their proposed solution is intractable for larger inputs and groups, because of the rapidly growing size of tensors and the number of operations needed for forward and backward passes in the network.

The aim of this paper is to propose a method that enables us to design a novel $G$-invariant architecture for a given finite group $G \leq S_n$, which is universal, able to generalize well and tractable even for big groups. The paper is organized as follows:
\vspace{-0.1in}
\begin{enumerate}[leftmargin=*]
\setlength\itemsep{0.0em}
    \item Related work is given in Section~\ref{rw}.
    \item In Section~\ref{arc}, we introduce our  $G$-invariant network architecture, which consists of (i) a $G$-invariant transformation block composed of a $G$-equivariant network and a Sum-Product Layer (SPL) denoted by $\Sigma\Pi$ which employs a superposition of product units \citep{product_units}, and (ii) a fully connected neural network.
    \item In Section~\ref{sec:GHS}, we elucidate the invariance of the proposed network to the actions of hierarchical subgroups. For this purpose, we describe the cases when the proposed {$G$-invariant} network can be also $H$-invariant for ${G < H \leq S_n}$.
    \item In Section~\ref{sec:universality}, we prove that any continuous $G$-invariant function $f:V \rightarrow \mathbb{R}$, where $V$ is a compact subset of $\mathbb{R}^{n \times n_{in}}$, for some $n, n_{in} > 0$, can be approximated using the proposed $G$-invariant network architecture.
    \item  In Section~\ref{sec:efficiency}, we discuss in detail the computational efficiency of the proposed method and relate that to the state-of-the-art $G$-invariant architecture proposed by \citet{maron19}.
    \item In Section~\ref{sec:experiments}, we provide experimental analyses and numerical evaluation of the proposed method and state-of-the-art $G$-invariant neural networks on two benchmark tasks: (i) $G$-invariant polynomial approximation and (ii) convex quadrangle area estimation. Moreover, we examine experimentally scalability, robustness and computational efficiency of the models learned using the proposed $G$-invariant networks.
    \item In Section~\ref{conc}, we summarize the paper and provide a detailed discussion.
\end{enumerate}

\section{Related Work}
\label{rw}

In order to make use of symmetry properties of data while learning deep feature representations, various $G$-invariant or $G$-equivariant neural networks have been proposed in the last decade. In various tasks, learned network models should reveal the invariance or equivariance to the whole group $S_n$ of all permutations on a finite set of  $n$-elements. \citet{pointnet} applied a permutation invariant network for point cloud processing, whereas \citet{deep_sets} applied both invariant and equivariant networks on sets. A permutation equivariant model was used by \citet{set_interact} to model interactions between two or more sets. For this purpose, they proposed a method to achieve permutation equivariance by parameter sharing. \cite{set_transformer} proposed an approach to achieve invariance to all permutations of input data utilizing an attention mechanism.
Another popular use case of $S_n$-invariance and equivariance properties are neural networks working on graphs, which were discussed in \citet{G_inv_graphs_1} and \citet{G_inv_graphs_2}.
Although the aforementioned papers present interesting approaches to obtain invariance to all permutations, the approach proposed in our paper allows to induce more general invariance to any subgroup of the symmetric group $S_n$.

$G$-equivariant neural networks, where $G$ is not a subgroup of $S_n$, are considered by \citet{cohen1} and \citet{cohen2}. $G$-equivariant Convolutional Neural Networks on homogeneous spaces were discussed by \citet{cohen2}, whereas \citet{cohen1} considered modeling invariants to actions of the groups on images, such as to image reflection and rotation.

Recent works have studied invariants to some specific finite subgroups $G$ of the symmetric group $S_n$, which is also considered in this paper. An approach exploiting the parameter sharing for achieving the invariant and equivariant models to such group actions was introduced by \citet{parameter_sharing}. \citet{maron19} used a linear layer model to compute a $G$-invariant and equivariant universal approximation function. However, their proposed solution requires the use of high dimensional tensors, which can be intractable for larger inputs and groups. In turn, \citet{yarotsky} considered provably universal architectures that are based on polynomial layers, but he assumed that the generating set of $G$-invariant polynomials is given, which is rather impractical. Moreover, there is also a simple approach to achieve $G$-invariance of any function, which exploits averaging of outputs of functions over a whole group $G$ \citep{reynolds}, but it linearly increases the overall number of computations with the size of the group.
 
The approach proposed in this paper builds on the work of \citet{yarotsky} and provides a network architecture to perform end-to-end tasks requiring $G$-invariance using a tractable number of parameters and operations, utilizing product units \citep{product_units} with Reynolds operator \citep{reynolds}.
While our approach is not dedicated to image processing and computer vision tasks, it can be used to construct $G$-invariant networks for different types of structured data that do not necessarily have temporal or sequential ordering (e.g. geometric shapes and graphs). This makes the proposed architecture useful in geometric deep learning \cite{bronstein}, and in a wide area of applications, from robotics to molecular biology and chemistry, where it can be used e.g. for estimating the potential energy surfaces of the molecule \citep{chemia1, chemia2}.

\section{$G$-invariant Network}

In this section, we introduce a novel $G$-invariant neural network architecture, which exploits the theory of invariant polynomials and the universality of neural networks to achieve a flexible scheme for $G$-invariant transformation of data for some known and finite group $G \leq S_{n}$, where $S_n$ is a symmetric group and $|G| = m$. Next, we discuss invariance of networks to actions of groups with a hierarchical structure, such as invariance to actions of groups $H$, where $G < H \leq S_n$. Then, we prove the universality of the proposed method and finally analyze its computational and memory complexity.

\subsection{$G$-invariant Network Architecture}
\label{arc}
We assume that an input $x \in \mathbb{R}^{n \times n_{in}}$ to the proposed network is a tensor\footnote{We use matrix notation to denote tensors in this paper.} $x = [x_1\, x_2\,  \dots\, x_n]^T$ of $n$ vectors $x_i \in \mathbb{R}^{n_{in}}, i =1,2,\ldots, n$. 
The $G$-invariance property of a function $f: \mathbb{R}^{n \times n_{in}} \rightarrow \mathbb{R}$  means that $f$ satisfies
\begin{equation}
\label{eq:Ginv}
  \allfor_{x \in \mathbb{R}^{n \times n_{in}}}\allfor_{g \in G}\,f(g(x)) = f(x), 
\end{equation}
where\footnote{$\allfor_{y \in Y} P(Y)$ means that ``predicate $P(Y)$ is true for all $y \in Y$''.} the action of the group element $g$ on $x$ is defined by
\begin{equation}
\label{eq:Gequiv}
    g(x) = \{x_{\sigma_g(1)}, x_{\sigma_g(2)}, \dots, x_{\sigma_g(n)}\},
\end{equation}
where $x_{\sigma_g(i)} \in \mathbb{R}^{n_{in}}$ and $\sigma_g(i)$ represents the action of the group element $g$ on the specific index $i$.
Similarly, a function $f: \mathbb{R}^{n \times p} \rightarrow \mathbb{R}^{n \times q}$ has a $G$-equivariance property, if the function $f$ satisfies
\begin{equation}
   \allfor_{g \in G} \allfor_{x \in \mathbb{R}^{n \times p}} \,g(f(x)) = f(g(x)).
\end{equation}

\begin{figure}[ht]
\vskip -0.1in
\begin{center}
\centerline{\includegraphics[width=\columnwidth]{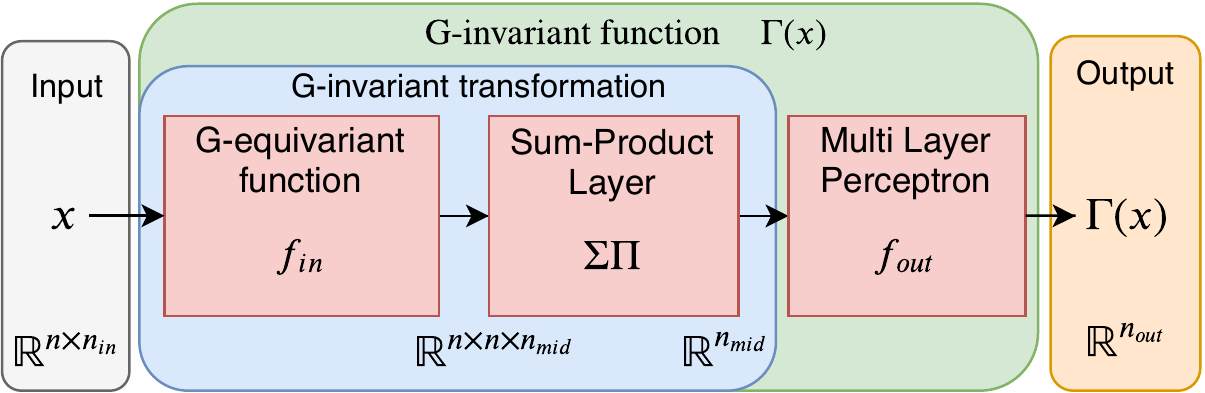}}
\caption{An illustration of the proposed $G$-invariant neural network. Input $x$ is processed by the $G$-invariant transformation (blue), which produces $G$-invariant representation of the input. Then, the $G$-invariant representation is passed to the Multi Layer Perceptron which produces the output vector $\Gamma(x)$.}
\label{fig:ginv}
\end{center}
\vskip -0.2in
\end{figure}
The proposed $G$-invariant neural network is illustrated in Figure \ref{fig:ginv} and defined as function $\Gamma: \mathbb{R}^{n \times n_{in}} \rightarrow \mathbb{R}^{n_{out}}$ of the following form
\begin{equation}
\label{eq:Gamma}
   \Gamma(x) = f_{out}(\Sigma\Pi(f_{in}(x))),
\end{equation}
where $f_{in}$ is a $G$-equivariant input transformation function, $\Sigma\Pi$ is a function which, when combined with $f_{in}$, comprises $G$-invariant transformation and $f_{out}$ is an output transformation function.
The general idea of the proposed architecture is to define a $G$-invariant transformation, which uses the sum of $G$-invariant polynomials ($\Sigma\Pi$) of $n$ variables, which are the outputs of $f_{in}$. This transformation produces a $G$-invariant feature vector, which is processed by another function $f_{out}$ that is approximated by the Multi-Layer Perceptron.

First, let us define the $G$-equivariant input transformation function $f_{in}: \mathbb{R}^{n \times n_{in}} \rightarrow \mathbb{R}^{n \times n \times n_{mid}}$, where $n_{mid}$ is the size of the feature vector. This function can be represented as a vector $\Phi = [\phi_1\, \phi_2\, \dots\, \phi_n]$ of neural networks, where each function $\phi_i: \mathbb{R}^{n_{in}} \rightarrow \mathbb{R}^{n_{mid}}$ is applied on all elements of the set of input vectors $\{x_i\}_{i=1}^n$, and transforms them to the $n_{mid}$ dimensional vector.
As a result, the operation of the $f_{in}$ function can be formulated by
\begin{equation}
\label{eq:fin}
    f_{in}(x) = 
    \begin{bmatrix}
        \Phi(x_1)\\
        \Phi(x_2)\\
        \vdots\\
        \Phi(x_n)\\
    \end{bmatrix} = 
    \begin{bmatrix}
        \phi_1(x_1) & \ldots & \phi_n(x_1)\\
        \vdots & \ddots & \vdots \\
        \phi_1(x_n) & \ldots & \phi_n(x_n)\\
    \end{bmatrix}.
\end{equation}
One can see that $f_{in}(x)$ is $G$-equivariant, since the action of the vector $\Phi$ of functions is the same for each element of the vector $x$, thus it transposes the rows of the matrix form \eqref{eq:fin} according to $g \in G$, which is equivalent to transposing the rows after the calculation of $f_{in}(x)$ by
\begin{equation}
{\small 
   f_{in}(g(x)) = 
    \begin{bmatrix}
        \Phi(x_{\sigma_g(1)})\\
        \Phi(x_{\sigma_g(2)})\\
        \vdots\\
        \Phi(x_{\sigma_g(n)})\\
    \end{bmatrix} = 
     g\left(\begin{bmatrix}
        \Phi(x_1)\\
        \Phi(x_2)\\
        \vdots\\
        \Phi(x_n)\\
    \end{bmatrix}\right) = g(f_{in}(x)).
}
\end{equation}
Second, we define the function $\Sigma\Pi: \mathbb{R}^{n \times n \times n_{mid}} \rightarrow \mathbb{R}^{n_{mid}}$, which constructs $G$-invariant polynomials of outputs obtained from $f_{in}$, by
\vspace{-0.1in}
\begin{equation}
\label{eq:sigmapi}
    \Sigma\Pi(x) = \sum_{g \in G}\prod_{j=1}^{n} x_{\sigma_g(j), j}.
\end{equation}
To see the $G$-invariance of $\Sigma\Pi(f_{in}(x))$, we substitute $x$ from \eqref{eq:sigmapi} with \eqref{eq:fin} to obtain
\vspace{-0.1in}
\begin{equation}
\label{eq:sigmapifin_G}
    \Sigma\Pi(f_{in}(x)) = \sum_{g \in G}\prod_{j=1}^{n}  \phi_j(x_{\sigma_g(j)}).
\end{equation}
Then, we can show that \eqref{eq:sigmapifin_G} is $G$-invariant by checking whether \eqref{eq:Ginv} holds for any input $x$ and any group element $g' \in G$ as follows:
\vspace{-0.05in}
\small
\begin{equation}
\begin{split}
    \Sigma\Pi(g'(f_{in}(x))) &= \sum_{g \in G}\prod_{j=1}^{n} \phi_j(x_{\sigma_{g'}(\sigma_g(j))})\\ 
    &= \sum_{g \in G}\prod_{j=1}^{n}  \phi_j(x_{\sigma_g(j)}) = \Sigma\Pi(f_{in}(x))
\end{split},
\end{equation}
\normalsize
since any group element acting on the group leads to the group itself.
Last, we define the output function $f_{out}: \mathbb{R}^{n_{mid}} \rightarrow \mathbb{R}^{n_{out}}$ following the structure of a typical fully connected neural network by
\vspace{-0.05in}
\begin{equation}
\label{eq:fout}
  f_{out}(x) = \sum_{i=1}^{N}c_i \sigma\left(\sum_{j=1}^{n_{mid}} w_{ij} x_j + h_i\right),
\end{equation}
where $N \in \mathbb{N}_{+}$ is a parameter, $\sigma$ is a non-polynomial activation function and $c_i, w_{ij}, h_i \in \mathbb{R}$ are coefficients.

\subsection{Invariance to Actions of Hierarchical Subgroups}
\label{sec:GHS}
Note that, it is possible to obtain a function of the form $\Gamma$ that is not only $G$-invariant, but also $H$-invariant, for some $G < H \leq S_n$. Such a case is in general contradictory to the intention of the network user, because it imposes more constraints than imposed by the designer of the network. 
To illustrate such a case, assume that
\begin{equation}
    \phi_i(x) = \phi(x)\quad \text{for} \quad {i \in \{1, 2, \dots, n\}}, 
\end{equation}
for an arbitrary function $\phi: \mathbb{R}^{n_{in}} \rightarrow \mathbb{R}^{n_{mid}}$.
Then, the action of the function $\Sigma\Pi(f_{in}(x))$ will be defined by
\begin{equation}
    \Sigma\Pi(f_{in}(x)) = \sum_{g \in G}\prod_{j=1}^{n}  \phi(x_{\sigma_g(j)}) = m \prod_{j=1}^{n}  \phi(x_j),
\end{equation}
which is both $G$-invariant and $S_n$-invariant.
So, it is clear that there exists some identifications of the form 
\begin{equation}
\allfor_{E \subset \mathcal{P}(\{0, 1, \dots, n\})}\allfor_{e \in E} 
\phi_e(x) = \phi_E(x),
\end{equation}
where $\mathcal{P}(X)$ denotes the power set of the set $X$, and $\phi_E: \mathbb{R}^{n_{in}} \rightarrow \mathbb{R}^{n_{mid}}$ is a function, which leads to the {$H$-invariance} for some $G < H \leq S_n$.

However, we conjecture that, if the function $f_{in}$ is realized by a randomly initialized neural network, then such identifications are almost impossible to occur and the function $\Gamma$ will be $G$-invariant only. 
But, if we consider a case when the data can reveal $H$-invariant models, then the proposed solution enables network models to learn identifications needed to achieve also $H$-invariance. This property is desirable since it at the same time retains the $G$-invariance and allows for stronger invariants if learned from data. 

\subsection{The Universality of the Proposed $G$-Invariant Network}
\label{sec:universality}
\begin{proposition}
\label{thm:prop}
The network function (\ref{eq:Gamma}), can approximate any $G$-invariant function $f: V \rightarrow \mathbb{R}$, where $V$ is a compact subset of $\mathbb{R}^{n \times n_{in}}$ and $G \leq S_n$ is a finite group, as long as number of features $n_{mid}$ at the output of input transformation network $f_{in}$ is greater than or equal to the size $N_{inv}$ of the generating set $\mathcal{F}$ of polynomial $G$-invariants.
\end{proposition}

\begin{proof}
In the proof, without the loss of generality, we consider the case when $n_{out} = 1$, as the approach can be generalized for arbitrary $n_{out}$. Moreover, we assume that 
\begin{equation}
\label{eq:V}
    0 \notin V
\end{equation}
to avoid the change of sign when approximating polynomials of inputs, but it is not a limitation because any compact set can be transformed to such a set by a bijective function.

To prove the Proposition \ref{thm:prop}, we need to employ two theorems:
\begin{theorem}[\citealp{yarotsky}]
\label{thm:UAT}
Let $\sigma: \mathbb{R} \rightarrow \mathbb{R}$ be a continuous activation function that is not a polynomial. Let $V=\mathbb{R}^d$ be a real finite dimensional vector space.
Then, any continuous map ${f:V \rightarrow \mathbb{R}}$ can be approximated, in the sense of uniform convergence on compact sets, by
\begin{equation}
\label{eq:pinkus}
  \hat{f}(x_1, x_2, \dots, x_d) = \sum_{i=1}^{N}c_i \sigma\left(\sum_{j=1}^{d} w_{ij} x_j + h_i\right)
\end{equation}
with a parameter $N \in \mathbb{N}_+$ and coefficients $c_i, w_{ij}, h_{i} \in \mathbb{R}$.
\end{theorem}
\vspace{-0.15in}
The above version of the theorem comes from the work of \citet{yarotsky}, but it was proved by \citet{pinkus}.

\begin{theorem}[\citealp{yarotsky}]
\label{thm:poly}
Let $\sigma: \mathbb{R} \rightarrow \mathbb{R}$ be a continuous activation function that is not a polynomial, $G$ be a compact group, $W$ be a finite-dimensional $G$-module and $f_1,\ldots,f_{N_{inv}}: W \rightarrow \mathbb{R}$ be a finite generating set of polynomial invariants on $W$ (existing by Hilbert’s theorem).
Then, any continuous invariant map $f:W \rightarrow \mathbb{R}$ can be approximated by an invariant map $\hat{f}:W \rightarrow \mathbb{R}$ of the form
\begin{equation}
\label{eq:yarotsky}
  \hat{f}(x) = \sum_{i=1}^{N}c_i \sigma\left(\sum_{j=1}^{N_{inv}} w_{ij} f_j(x) + h_i\right)
\end{equation}
with a parameter $N \in \mathbb{N}_+$ and coefficients $c_i, w_{ij}, h_{i} \in \mathbb{R}$.
\end{theorem}
\vspace{-0.1in}
The accuracy of the approximation (\ref{eq:yarotsky}) has been proven to be $2\epsilon$ for some arbitrarily small positive constant $\epsilon$. 
Note that the function $f_{out}$ \eqref{eq:fout}, is of the same form as the function $\hat{f}$ \eqref{eq:yarotsky}. Then, one can accurately imitate the behavior of $\hat{f}$ using $f_{out}$, if the input to both functions are equivalent.

\begin{lemma}
\label{thm:lemma}
For every element $f_i: V \rightarrow \mathbb{R}$ of the finite generating set $\mathcal{F} = \{ f_i\}_{i=1} ^{N_{inv}}$ of polynomial $G$-invariants on $V$, there exists an approximation of the form (\ref{eq:sigmapifin_G}), linearly dependent on $\epsilon$, where $G \leq S_n$ is an $m$ element subgroup of the $n$ element permutation group and $\epsilon$ is an arbitrarily small positive constant.
\end{lemma}

\begin{proof}
Any function $f_i \in \mathcal{F}$ has the following form
\begin{equation}
\label{eq:inv_f_i_sigma}
    f_i(x) = \sum_{g \in G} \psi(g(x)),
\end{equation}
where
\begin{equation}
\label{eq:inv_f_i_pi}
    \psi(x) = \prod_{i=1}^{n} x_i^{b_i},
\end{equation}
and $b_i$ are fixed exponents.
Combining (\ref{eq:inv_f_i_sigma}) and (\ref{eq:inv_f_i_pi}), we obtain:
\begin{equation}
\label{eq:inv_f_i}
    f_i(x) = \sum_{g \in G} \prod_{i=1}^{n} x_{\sigma_g(i)}^{b_i},
\end{equation}
which has a similar form as \eqref{eq:sigmapifin_G}.
This resemblance is not accidental, but in fact, $\Sigma\Pi(f_{in}(x))$ can
approximate $n_{mid}$ functions belonging to the set $\mathcal{F}$.
Using Theorem \ref{thm:UAT} and the fact that $\phi_i$ is a neural network satisfying (\ref{eq:pinkus}), we observe that $\phi_j(x_i)$ can approximate any continuous function with $\epsilon$ precision. Thus, it can approximate $x_{i}^{b_i}$ for some constant parameter $b_i$. It is possible to provide an upper bound on the approximation error $\left| f_i(x) - \Sigma\Pi_i(f_{in}(x)) \right|$
by
\begin{equation}
\label{eq:sigmapifin_err}
\begin{split}
    &\left| f_i(x) - \Sigma\Pi_i(f_{in}(x)) \right| \stackrel{(\ref{eq:inv_f_i}, \ref{eq:sigmapifin_G})}{=} \\
    &\left| \sum_{g \in G} \prod_{i=1}^{n} x_{\sigma_g(i)}^{b_i} - \sum_{g \in G}\prod_{j=1}^{n}  \phi_j(x_{\sigma_g(j)}) \right| \leq\\
    & \sum_{g \in G} \left|\prod_{i=1}^{n} x_{\sigma_g(i)}^{b_i} - \prod_{j=1}^{n}  \phi_j(x_{\sigma_g(j)}) \right| \leq\\
    & \sum_{g \in G} \left|\prod_{i=1}^{n} x_{\sigma_g(i)}^{b_i} - \prod_{j=1}^{n}  (x_{\sigma_g(j)}^{b_j} - \epsilon) \right| \stackrel{(\ref{eq:V})}{\leq} mn\epsilon\\
\end{split},
\end{equation}
for some arbitrarily small positive constant $\epsilon$.
\end{proof}
Assuming that the number of features $n_{mid}$ at the output of input transformation network $f_{in}$ is greater than or equal to the size of the generating set $\mathcal{F}$, it is possible to estimate each of $f_i(x)$ functions using \eqref{eq:sigmapifin_G}.

The last step for completing the proof of the Proposition \ref{thm:prop}, using Theorem \ref{thm:UAT}, Theorem \ref{thm:poly}, and the proposed Lemma \ref{thm:lemma}, is to show that
\begin{equation}
    \left| f(x) - \Gamma(x) \right| \leq \epsilon c,
\end{equation}
where $c \in \mathbb{R}$ is a constant.

Let us consider the error
\begin{equation}
\label{eq:last_bound}
\begin{split}
    & \left| f(x) - \Gamma(x) \right| \stackrel{(\ref{eq:Gamma})}{=} \left| f(x) - \hat{f}(x)\right| + \\
    & + \left|\hat{f}(x) - f_{out}(\Sigma\Pi(f_{in}(x))) \right| \stackrel{\text{Thm.} \ref{thm:poly}}{=} \\
    & 2\epsilon + \left|\hat{f}(x) - f_{out}(\Sigma\Pi(f_{in}(x))) \right| = \\
    & \left.  f_{out}(\mathcal{F}(x)) - f_{out}(\Sigma\Pi(f_{in}(x))) \right| \leq \\
    & 2\epsilon + \left|\hat{f}(x) - f_{out}(\mathcal{F}(x)) \right| + \\
    & \left| f_{out}(\mathcal{F}(x)) - f_{out}(\Sigma\Pi(f_{in}(x))) \right| \stackrel{\text{Thm.} \ref{thm:UAT}, (\ref{eq:fout})}{\leq} \\
    & 3\epsilon + \left| f_{out}(\mathcal{F}(x)) - f_{out}(\Sigma\Pi(f_{in}(x))) \right|
\end{split}.
\end{equation}
Several transformations presented in (\ref{eq:last_bound}) result in the formula which is a sum of $3\epsilon$ and the absolute difference of $f_{out}(\mathcal{F})$ and $f_{out}(\Sigma\Pi(f_{in}(x)))$. From (\ref{eq:sigmapifin_err}), we have that the difference of the arguments is bounded by $mn\epsilon$. Consider then a ball $B_{mn\epsilon}(x)$ with radius $mn\epsilon$ centered at $x$. Since $f_{out}$ is a MLP (multi-layer perceptron), which is at least locally Lipschitz continuous, we know that its output for $x' \in B_{mn\epsilon}(x)$ can change at most by $kmn\epsilon$, where $k$ is a Lipschitz constant. From those facts, we can provide an upper bound on the error (\ref{eq:last_bound}) by
\begin{equation}
\begin{split}
    & \left| f(x) - \Gamma(x) \right| = \\
    & 3\epsilon + \left| f_{out}(\mathcal{F}(x)) - f_{out}(\Sigma\Pi(f_{in}(x))) \right| \leq \\
    & 3\epsilon + kmn\epsilon = \epsilon(3 + kmn) = \epsilon c \\
\end{split}.
\end{equation}
\end{proof}

\subsection{Analysis of Computational and Memory Complexity}
\label{sec:efficiency}
Having proved that the proposed approach is universal we elucidate its computational and memory complexity.

The tensor with the largest size is obtained at the output of the $f_{in}$ function. The size of this tensor is equal to $n^2 n_{mid}$, where we assume that $n_{mid} \geq N_{inv}$ and it is a design parameter of the network. So, the memory complexity is of the order $n^2 n_{mid}$, which is polynomial. However, the complexity of the method proposed by \citet{maron19}, is of the order $n^p$, where $\frac{n-2}{2} \leq p \leq \frac{n(n-1)}{2}$ depending on the group $G$.

In order to evaluate the function $\Sigma\Pi$, $m (n-1) n_{mid}$ multiplications are needed, where $m = |G|$ and $n_{mid}$ is a parameter, but we should assure that $n_{mid} \geq N_{inv}$ to ensure universality of the proposed method (see Section 3.3).
It is visible, that the growth of the number of computations is linear with $m$.
For smaller subgroups of $S_n$, such as $\mathbb{Z}_n$ or $D_{2n}$, where $m \propto n$, the number of the multiplications is of order $n^2$, which is a lot better than the number of multiplications performed by the $G$-invariant neural networks proposed in \citet{maron19}, which is of order $n^p$. 
However, for big groups, where $m$ approaches $n!$, the number of multiplications increases. 
Although the proposed approach can work for all subgroups of $S_n$ ($m$=$n!$), it suits the best for smaller, yet not less important, groups such as cyclic groups $Z_n$, $D_{2n}$, $S_k$ ($k < n$) or their direct products.

Moreover, the proposed $\Sigma\Pi$ can be implemented efficiently on GPUs using a parallel implementation of matrix multiplication and reduction operations in practice. Thereby, we obtain almost similar running time for increasing $n_{mid}$ and $m$ in the experimental analyses given in the next section.

\section{Experimental Analyses}
\label{sec:experiments}
\subsection{Definitions of Tasks}
We evaluate the accuracy of the proposed method and analyze its invariance properties in the following two tasks.

\subsubsection{$G$-invariant Polynomial Regression}
The goal of this task is to train a model to approximate a $G$-invariant polynomial. In the experiments, we consider various polynomials: $P_{\mathbb{Z}_k}, P_{S_k}, P_{D_{2k}}, P_{A_k}$ and $P_{S_k \times S_l}$, which are invariant to the cyclic group $\mathbb{Z}_k$, permutation group $S_k$, dihedral group $D_{2k}$, alternating group $A_k$ and direct product of two permutation groups $S_k \times S_l$, respectively. The formal mathematical definitions of those polynomials are given in the appendix \ref{sec:ds}. To examine generalization abilities of the proposed $G$-invariant network architecture, the learning was conducted using only 16 different random points in $[0; 1]^5$, whereas 480 and 4800 randomly generated points were used for validation and testing, respectively.

\subsubsection{Estimation of Area of Convex Quadrangles}
In this task, models are trained to estimate areas of convex quadrangles. An input is a vector of 4 points lying in $\mathbb{R}^{4 \times 2}$, each described by its $x$ and $y$ coordinates. Note that shifting the sequence of points does not affect the area of the quadrangle (we assume that reversing the order does, but such examples do not occur in the dataset, so it can be neglected). The desired estimator is a simple example of the $G$-invariant function, where $G = \mathbb{Z}_4 = (1234)$.
In the experiments, both training and validation set contains 256 examples (randomly generated convex quadrangles with their areas), while the test dataset contains 1024 examples. Coordinates of points take values from $[0; 2]$, whereas areas take value from $(0; 1]$. More detailed information about the proposed datasets can be found in the appendix \ref{sec:ds} and code\footnote{\label{code}\url{https://github.com/Kicajowyfreestyle/G-invariant}}.

\subsection{Compared Architectures and Models}
All of the experiments presented below consider networks of different architectures for which the number of weights was fixed at a similar level for the given task, to obtain fair comparison,
The considered architectures are the following:
\vspace{-0.1in}
\begin{itemize}
\setlength\itemsep{0.0em}
    \item FC $G$-avg: Fully connected neural network with Reynolds operator \cite{reynolds},
    \item Conv1D $G$-avg: 1D convolutional neural network with Reynolds operator,
    \item FC $G$-inv: $G$-invariant neural network (\ref{eq:Gamma}) implementing $f_{in}$ using a fully connected neural network,
    \item Conv1D $G$-inv: $G$-invariant neural network (\ref{eq:Gamma}) implementing $f_{in}$ using 1D Convolutional Neural Network,
    \item Maron: $G$-invariant network \cite{maron19}.
\end{itemize}
All of those functions are used in both tasks and differ between the tasks only in the number of neurons in some layers. More detailed information about the aforementioned architectures is included in the appendix \ref{sec:models} and code\textsuperscript{\ref{code}}.

Moreover, for all experiments, both running times and error values are reported by calculating their mean and standard deviation over 10 independent models using the same architecture, chosen by minimal validation error during  training, to reduce impact of initialization of weights.

\begin{table*}[ht]
\vskip -0.1in
\caption{Mean absolute errors (MAEs) [$10^{-2}$] of several $G$-invariant models for the task of $G$-invariant polynomial regression.}
\label{tab:poly}
\begin{center}
\begin{small}
\begin{sc}
\begin{tabular}{lcccc}
\toprule
Network & Train & Validation & Test & \#Weights [$10^3$]\\
\midrule
FC $G$-avg              & 15.15 $\pm$ 5.49  & 16.48 $\pm$ 0.73  & 16.89 $\pm$ 0.76   & 24.0\\
\textbf{$G$-inv (ours) }         & 2.65 $\pm$ 0.91   & 7.32 $\pm$ 0.55   & 7.46  $\pm$ 0.56   & 24.0\\
Conv1D $G$-avg & 8.98 $\pm$ 6.39 & 11.43 $\pm$ 4.29 & 11.78 $\pm$ 4.79 & 24.0\\
\textbf{Conv1D $G$-inv (ours) }  & \textbf{0.87 $\pm$ 0.12} & \textbf{2.57 $\pm$ 0.37} & \textbf{2.6 $\pm$ 0.4} & 24.0\\
Maron                   & 2.41 $\pm$ 0.82   & 5.74 $\pm$ 1.19   & 5.93  $\pm$ 1.18   & 24.2\\
\bottomrule
\end{tabular}
\end{sc}
\end{small}
\end{center}
\vskip -0.2in
\end{table*}

\subsection{Results for $\mathbb{Z}_5$-invariant Polynomial Regression}
\label{sec:poly}
In the task of $\mathbb{Z}_5$-invariant polynomial regression, the training lasts for 2500 epochs, after which only slight changes in the accuracy of the models were reported. We measure accuracy of the models using mean absolute error (MAE) defined by \citet{MAE}.
The accuracy of the examined models is given in Table \ref{tab:poly}.

We observe that our proposed Conv1D $G$-inv outperforms all of the other architectures on both datasets. Both Maron and FC $G$-inv obtain worse MAE, but they significantly outperform the Conv1D $G$-avg and FC $G$-avg.
Moreover, those architectures obtain large standard deviations for the training dataset, because sometimes they converge to different error values. In contrast, the performance of the $G$-inv based models and the Maron model is relatively stable under different weight initialization.

While the results are similar for our proposed architecture and the approach introduced in \citet{maron19}, the number of computations needed to train and evaluate the Maron model is significantly larger compared to our $G$-invariant network. The inference time for both networks differs notably, and equals $2.3 \pm 0.4$ms for Conv1D $G$-inv and $21.4 \pm 1.5$ms for Maron, where the evaluation of those times was performed on 300 inferences with batch size set to 16 using an Nvidia GeForce GTX1660Ti.

\begin{table*}[t]
\caption{Mean absolute errors (MAEs) [$10^{-3} \text{unit}^2$] of several $G$-invariant models for the task of convex quadrangle area estimation.}
\label{tab:area}
\begin{center}
\begin{small}
\begin{sc}
\begin{tabular}{lcccc}
\toprule
Network & Train & Validation & Test & \#Weights\\
\midrule
FC $G$-avg              & 7.0 $\pm$ 0.6 & 9.6 $\pm$ 1.0 & 9.4 $\pm$ 0.9   & 1765\\
\textbf{$G$-inv (ours)} & 7.4 $\pm$ 0.4 & 8.0 $\pm$ 0.3 & 8.3 $\pm$ 0.5   & 1785\\
Conv1D $G$-avg          & 16.9 $\pm$ 7.7 & 16.8 $\pm$ 5.3 & 18.5 $\pm$ 6.8 & 1667\\
\textbf{Conv1D $G$-inv (ours)}  & \textbf{6.0 $\pm$ 0.3} & \textbf{7.3 $\pm$ 0.3} & \textbf{7.5 $\pm$ 0.5}   & 1673\\
Maron                   & 13.9 $\pm$ 0.9 & 22.3 $\pm$ 1.2 & 23.4 $\pm$ 1.3   & 1802\\
\bottomrule
\end{tabular}
\end{sc}
\end{small}
\end{center}
\vskip -0.2in
\end{table*}

\subsection{Results for Estimation of Areas of Convex Quadrangles}
\label{sec:quads}
In the task of estimating areas of convex quadrangles, each model was trained for 300 epochs and the accuracy of the models on training, validation and test sets are reported in Table \ref{tab:area}.
The results show that the model utilizing the approach presented in this paper obtains the best performance on all three datasets. 
Furthermore, it generalizes much better to the validation and test dataset than any other tested approach. However, one has to admit that the differences between $G$-inv models and fully connected neural network exploiting Reynolds operator (FC $G$-avg) are relatively small for all three datasets.
We observe that, besides the proposed $G$-invariant architecture, the only approach which was able to reach a low level of MAE in the polynomial approximation task (Maron) is unable to accurately estimate the area of the convex quadrangle, which is a bit more abstract task, possibly not easily translatable to some $G$-invariant polynomial regression.

\subsection{Analysis of the Effect of the Group Size on the Performance}
\label{sec:group_size}
The goal of this experiment is to asses how the performance of the FC $G$-inv model changes with increasing size of a given group. To evaluate that, an approximation of several $G$-invariant polynomials was realized (in the same setup as for $\mathbb{Z}_5$-invariant polynomial regression, see Section \ref{sec:poly}). We measure accuracy of models using mean absolute percentage error (MAPE) defined by \citet{MAPE}.

The results on the test dataset are reported in Table \ref{tab:ngons}. The results show that while the upper bound of approximation error grows with the size of the group $m$, the error in the experiment exposes more complicated behavior. We observe that also the polynomial form affects the performance. For example, $P_{A_4}$ seems to be relatively easy to approximate using the proposed neural network. However, if we neglect $P_{A_4}$, the MAPE increases with the $m$, but slower than linear.
The evaluation times of the neural networks are independent from the group size, due to the ease of parallelization of the most expensive operation $\Sigma\Pi$, in which the number of multiplications grows linearly with $m$.

\begin{table}[t]
\vskip -0.1in
\caption{Mean absolute percentage errors (MAPEs) [\%] and inference times [ms] for the task of $G$-invariant polynomial approximation using FC $G$-inv model, for a few groups of different sizes.}
\label{tab:ngons}
\begin{center}
\begin{small}
\begin{sc}
\begin{tabular}{lcccc}
\toprule
 & $|G|$ & Train & Test & Time \\
\midrule
$P_{Z_5}$ & 5 & 3.2 $\pm$ 0.8 & 12.8 $\pm$ 4.6 & 2.3 $\pm$ 0.4\\
$P_{D_8}$ & 8 & 3.9 $\pm$ 1.7 & 10.4 $\pm$ 2.8 & 2.2 $\pm$ 0.2\\
$P_{A_4}$ & 12 & 2.5 $\pm$ 0.7 & 4.7 $\pm$ 1.1 & 2.3 $\pm$ 0.3\\
$P_{S_4}$ & 24 & 5.6 $\pm$ 2.7 & 14.9 $\pm$ 5.9 & 2.4 $\pm$ 0.4\\
\bottomrule
\end{tabular}
\end{sc}
\end{small}
\end{center}
\vskip -0.25in
\end{table}

\subsection{Analysis of the Effect of the Latent Space Size on the Performance}
\label{sec:latent_size}

In this experiment, we evaluate how the size of the $G$-invariant latent space $n_{mid}$ affects the MAE and inference time of both FC $G$-inv and Conv1D $G$-inv architectures. Those architectures were tested on the task of convex quadrangle area estimation for $n_{mid} \in \{1,2,8,32,128\}$, without changing the remaining parts of the networks.

Results given in Table \ref{tab:nmid_area} show that even low-dimensional $G$-invariant latent representation enables the network to estimate the area in the considered tasks.
While the accuracy of the Conv1D $G$-inv is almost the same regardless of the latent space size, the accuracy of FC $G$-inv improves significantly for $n_{mid}$ growing from 1 to 8.
Another interesting observation is that the inference time is independent of $n_{mid}$, which is achieved by using parallel computations on GPUs.

\begin{table}[ht]
\caption{Mean absolute errors (MAEs) [$10^{-3}$] and inference time [ms] on the test dataset for the task of convex quadrangle area estimation for different values of $n_{mid}$.}
\label{tab:nmid_area}
\begin{center}
\begin{small}
\begin{sc}
\begin{tabular}{lcccc}
\toprule
 & \multicolumn{2}{c}{Conv1D $G$-inv} & \multicolumn{2}{c}{FC $G$-inv}\\

$n_{mid}$ & MAE & Time & MAE & Time\\
\midrule

1 & 7.6 $\pm$ 0.3 & 2.9 $\pm$ 0.1 & 32.5 $\pm$ 0.7 & 3.0 $\pm$ 0.3 \\
2 & 7.5 $\pm$ 0.5 & 2.9 $\pm$ 0.1 & 10.1 $\pm$ 3.7 & 3.0 $\pm$ 0.2 \\
8 & 7.5 $\pm$ 0.4 & 2.9 $\pm$ 0.2 & 8.5 $\pm$ 0.3 & 2.9 $\pm$ 0.2 \\
32 & 7.3 $\pm$ 0.3 & 3.1 $\pm$ 0.7 & 8.1 $\pm$ 0.3 & 3.0 $\pm$ 0.1 \\
128 & 7.4 $\pm$ 0.3 & 2.9 $\pm$ 0.1 & 8.2 $\pm$ 0.4 & 3.4 $\pm$ 0.5 \\

\bottomrule
\end{tabular}
\end{sc}
\end{small}
\end{center}
\vskip -0.1in
\end{table}

\subsection{Robustness to Inaccurate Network Design}
\label{sec:robust}

We analyze performance of the FC $S_3$-inv network for $G$-invariant polynomial approximation, where $G \in \{\mathbb{Z}_3, S_3, S_3 \times S_2\}$ and $\mathbb{Z}_3 \le S_3 \le S_3 \times S_2 \le S_5$. 
The goal of the experiment is to assess the robustness of the proposed architecture to inaccurate network design, and validate the claims proposed in Section \ref{sec:GHS}, namely that the proposed $G$-invariant network is able to adjust to become approximately $H$-invariant, if the data expose the $H$-invariance, for $G < H \leq S_n$.

Figure \ref{fig:EGH} shows training and validation mean absolute percentage error (MAPE) computed during training of the same $S_3$-invariant model for learning to approximate $\mathbb{Z}_3$, $S_3$, $S_3 \times S_2$-invariant polynomials. The learning curves show that the proposed architecture is able to achieve the same level of accuracy when the approximated polynomial is $S_3$ or $S_3 \times S_2$-invariant. However, it is unable to reach that level for the $\mathbb{Z}_3$-invariant polynomial. The results confirm our claim that models, which are invariant to actions of an over-group $H$, can be learned from data using the proposed $G$-invariant network. Moreover, one can see that the $G$-invariant network is unable to adjust to the $E$-invariant data, where $E < G$, because it is unable to differentiate between data permuted with the element $g \in G \wedge g \notin E$. 
\begin{figure}[t]
\begin{center}
\centerline{\includegraphics[width=\columnwidth]{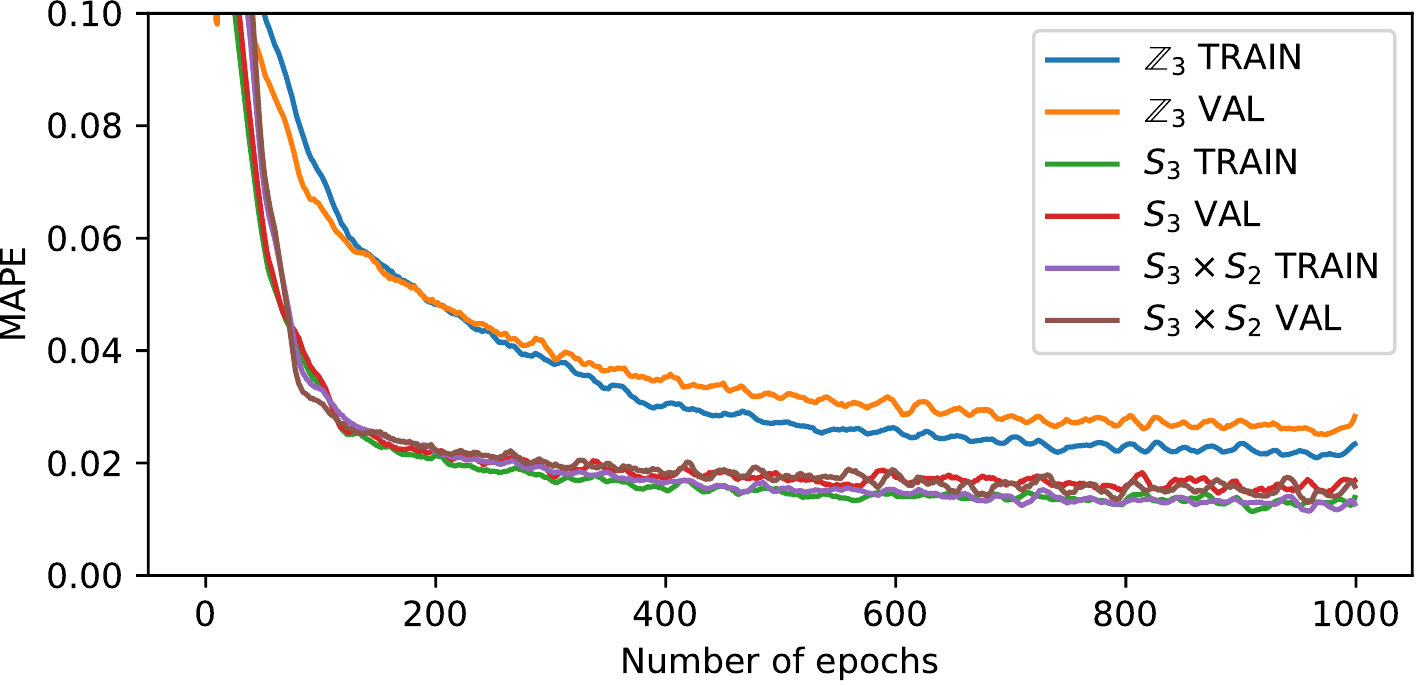}}
\vskip -0.05in
\caption{Learning curves of FC $S_3$-inv approximating $S_3 \times S_2$, $S_3$ and $\mathbb{Z}_3$ invariant polynomials. Even for the group $S_3\times S_2 > S_3$, $S_3$-invariant network is able to reach the same mean absolute percentage error (MAPE) as for $S_3$, for which the network was designed. However, it is unable to reach similar performance for $\mathbb{Z}_3$- invariant polynomial, because $S_3$-invariant network cannot differentiate between some permutations, which are not in $\mathbb{Z}_3$.}
\label{fig:EGH}
\end{center}
\vskip -0.35in
\end{figure}

\section{Discussion and Conclusion}
\label{conc}
In this paper, we have proposed a novel $G$-invariant neural network architecture that uses two standard neural networks, connected with the proposed Sum-Product Layer denoted by $\Sigma\Pi$. We have shown that the proposed architecture is a universal approximator as long as the number of features $n_{mid}$ at the output of input transformation network $f_{in}$ is greater than or equal to the size of the generating set $\mathcal{F}$ of polynomial $G$-invariants. Moreover, we analyzed the cases where the proposed network can obtain $H$-invariance properties for hierarchical groups $G < H \leq S_n$. We conjecture that it is challenging to obtain a $H$-invariant model using a randomly initialized $G$-invariant network unless the training data reveal $H$-invariance property. The ability of the $G$-invariant network to learn the $H$-invariance from data was experimentally verified in Section \ref{sec:GHS}.

We have also analyzed the computational efficiency of the proposed $G$-invariant neural network and compared it with the state-of-the-art $G$-invariant neural network architecture, which was proven to be universal. Analysis of the proposed network led us to the memory complexity of order $n^2 n_{mid}$ and computational complexity of order $m n n_{mid}$. Those polynomial dependencies suggest that the proposed approach is efficient and tractable, but it needs to be emphasized that the computational complexity can be cumbersome to handle for big groups such as $S_n$ or $A_n$, where $m \propto n!$. To support those considerations, inference times were reported for both tasks (see Table \ref{tab:ngons} and Table \ref{tab:nmid_area}). Interestingly, those running times are independent of $n_{mid}$ and $m$ due to the parallelization of the $\Sigma\Pi$ function.

Finally, we have conducted several experiments to explore various properties of the proposed $G$-invariant architecture in comparison with the other $G$-invariant architectures proposed in the literature. For this purpose, we used two tasks; (i) convex quadrangle area estimation and (ii) $G$-invariant polynomial regression. The results demonstrate that the proposed $G$-invariant neural network outperforms all other approaches in both tasks, no matter if it utilizes fully connected or convolutional layers. However, the Maron \cite{maron19} outperformed the $G$-inv neural network endowed with fully connected layers for polynomial regression. Note that, inference time of the Maron is an order of magnitude higher than that of the proposed method. 
It is also worth noting that employing convolutional layers for feature extraction in lower layers improves the accuracy of the whole architecture, probably by exploiting the intrinsic structure of the input data, such as neighborhood relations.

Furthermore, we analyzed the change of accuracy of the learned models depending on the latent vector size $n_{mid}$. The results pointed out that the proposed tasks can be solved using models with small $G$-invariant latent vectors, and that their inference time is nearly independent of the vector size, due to the easily parallelizable structure of the proposed $G$-invariant network.

We believe that the proposed $G$-invariant neural networks can be employed by researchers to learn group invariant models efficiently in various applications in machine learning, computer vision and robotics. In future work, we plan to apply the proposed networks for various tasks in robot learning, such as for path planning by vector map processing using the geometric structure of data.




\bibliography{example_paper}
\bibliographystyle{icml/icml2020}

\clearpage
\appendix
\section{Mathematical definitions}
\label{sec:math}

A \textbf{group} is a non-empty set $G$ with the binary operator ${\circ: G \times G \rightarrow G}$ called product, such that
\begin{enumerate}[leftmargin=*]
\setlength\itemsep{0.0em}
    \item $a, b \in G \implies a \circ b \in G$ (closed under product),
    \item $a, b, c \in G \implies (a \circ b) \circ c = a \circ (b \circ c)$ (associative),
    \item $\xeist_{e \in G}\allfor_{a \in G} a \circ e = e \circ a = a$ (existence of identity element),
    \item $\allfor_{a \in G}\xeist_{a^{-1} \in G} a \circ a^{-1} = a^{-1} \circ a = e$ (existence of inverse element).
\end{enumerate}

A \textbf{subgroup of the group $G$} is a non-empty subset $S \subset G$, which together with the product $\circ$, associated with the group $G$, forms a group.

A \textbf{permutation group} is a group whose elements are permutations.

\section{Parameters Used in the Experiments}
All experiments reported in the paper were performed using Nvidia GeForce GTX-1660 Ti with a learning rate equal to $10^{-3}$ and regularization parameter of the $\ell_2$ regularization was set to $10^{-5}$.

\section{Architectures Considered in the Experiments}
\label{sec:models}
In this section, we describe all neural network based models that were used in the experiments for comparative analysis of architectures.
It is worth noting that each of these neural networks uses the $\tanh$ activation function in its hidden layers,
except the output layers and layers right before the $G$-invariant latent representation, which do not use activation functions.

\textbf{FC $G$-avg} is an abbreviation of a fully connected neural network aggregated by the group averaging or more specifically Reynolds operator defined by
\begin{equation}
\label{eq:reynolds}
    f_R(x) = \frac{1}{|G|}\sum_{g \in G} f\left(g(x)\right),
\end{equation}
where $G$ is a finite group and $|G|$ denotes the size of the group (number of its elements). 
Hyperparameters of architectures of the networks used in experiments described in Section \ref{sec:poly} and \ref{sec:quads} of the main paper are given in Table \ref{tab:FCGavg}.
For both architectures, an output of a network is an average of forward passes for all $g \in G$ acting on the input of the network, according to \eqref{eq:reynolds}.

\begin{table}[htbp!]
\vspace{-0.1in}
\caption{Hyperparameters (FC number of kernels) of the architectures of the FC $G$-avg networks used in both the $\mathbb{Z}_5$-invariant polynomial approximation and the convex quadrangle area estimation experiments described in Section \ref{sec:poly} and \ref{sec:quads} of the main paper.}
\label{tab:FCGavg}
\begin{center}
\begin{small}
\begin{sc}
\begin{tabular}{cc}
\toprule
Polynomial approximation & Area estimation\\
\midrule
FC 89 & Flatten\\
FC 192 & FC 64 \\
FC 32 & FC 18 \\
FC 1 & FC 1 \\
\bottomrule
\end{tabular}
\end{sc}
\end{small}
\end{center}
\vskip -0.05in
\end{table}

\begin{table*}[t]
\caption{Hyperparameters of architectures of the Conv1D $G$-avg networks used in both $\mathbb{Z}_5$-invariant polynomial approximation and convex quadrangle area estimation experiments described in Section \ref{sec:poly} and \ref{sec:quads} of the main paper.}
\label{tab:Conv1DGavg}
\begin{center}
\begin{small}
\begin{sc}
\begin{tabular}{cc}
\toprule
Polynomial approximation & Area estimation \\
\midrule
Conv1D layer: 32 kernels of size 3x1 & Conv1D layer: 32 kernels of size 3x1 \\
Conv1D layer: 118 kernels of size 1x1 & Conv1D layer: 2 kernels of size 1x1 \\
Flatten Layer & Flatten Layer \\
FC layer: 32 output channels & FC layer: 32 output channels \\
FC layer: 1 output channel & FC layer: 1 output channel \\
\bottomrule
\end{tabular}
\end{sc}
\end{small}
\end{center}
\vskip -0.15in
\end{table*}

\textbf{Conv1D $G$-avg} is an abbreviation of a composition of a 1D convolutional neural network with a fully connected neural network and the group averaging defined in (\ref{eq:reynolds}).
It uses 1D convolutions to preprocess the input exploiting the knowledge about the group $G$,
namely, it performs the cyclic convolution on the graph imposed by the group $G$ -- each kernel acts on a triplet of the selected vertex and its two neighbors in terms of group operation.
Architectures of the networks used in experiments described in Section \ref{sec:poly} and \ref{sec:quads} of the paper are given in Table \ref{tab:Conv1DGavg}.
Similar to the FC $G$-avg, the output of a network is an average of the forward passes for all $g \in G$ acting on the input, according to \eqref{eq:reynolds}.
For the sake of implementation, the first and last elements of the input sequence are concatenated with the original input at the end and beginning respectively,
in order to use a typical implementation of convolutional neural networks (as they normally do not perform cyclic convolution).
For example, if the original input sequence looks like $[A\,B\,C\,D]$, then the network is supplied with sequence $[D\,A\,B\,C\,D\,A]$.

\textbf{FC $G$-inv} is an abbreviation of the $G$-invariant neural network equipped with a fully connected neural network
implementing an $f_{in}$ function proposed in the main paper. The general scheme of the $G$-invariant fully connected network architecture is described in Table \ref{tab:FCGinv}.
Values of $n$, $n_{in}$ and $n_{mid}$ differ between experiments and are listed in Table \ref{tab:FCGinv_n}.

\begin{table}[htbp!]
\caption{Hyperparameters of the FC $G$-inv architecture proposed in the main paper.}
\label{tab:FCGinv}
\begin{center}
\begin{small}
\begin{sc}
\begin{tabular}{cc}
\toprule
Layer & Output size \\
\midrule
Input & $n \times n_{in}$\\
FC & $n \times 16$\\
FC & $n \times 64$\\
FC & $n \times n n_{mid}$ \\
Reshape & $n \times n \times n_{mid}$ \\
$\Sigma\Pi$ & $n_{mid}$ \\
FC & 32 \\
FC & 1 \\
\bottomrule
\end{tabular}
\end{sc}
\end{small}
\end{center}
\end{table}

\begin{table}[htbp!]
\vspace{-0.2in}
\caption{Values of $n$, $n_{in}$ and $n_{mid}$ used for different experiments.}
\label{tab:FCGinv_n}
\begin{center}
\begin{small}
\begin{sc}
\begin{tabular}{cccc}
\toprule
Experiment (Section) & $n$ & $n_{in}$ & $n_{mid}$ \\
\midrule
\ref{sec:poly} & 5 & 1 & 64\\
\ref{sec:quads} & 4 & 2 & 2\\
\ref{sec:group_size} & 5 & 1 & 2\\
\ref{sec:latent_size} & 4 & 2 & \{1, 2, 8, 32, 128\}\\
\ref{sec:robust} & 5 & 1 & 8\\
\end{tabular}
\end{sc}
\end{small}
\end{center}
\end{table}

\textbf{Conv1D $G$-inv} is an abbreviation of the $G$-invariant neural network equipped with a 1D convolutional neural network implementing an $f_{in}$ function proposed in our paper. The general scheme of the $G$-invariant network architecture with a convolutional feature extractor is described in Table \ref{tab:Conv1DGinv}. Values of $n$, $n_{in}$ and $n_{mid}$ differ between experiments and are the same as for the FC $G$-inv model (listed in Table \ref{tab:FCGinv_n}), except the $n_{mid}$ used in the experiment given in Section \ref{sec:poly}, where $n_{mid} = 118$.

\begin{table}[htbp!]
\caption{Hyperparameters of the Conv1D $G$-inv architecture proposed in the main paper.}
\label{tab:Conv1DGinv}
\begin{center}
\begin{small}
\begin{sc}
\begin{tabular}{cc}
\toprule
Layer & Output size \\
\midrule
Input & $(n+2) \times n_{in}$\\
Conv1D 3x1 & $n \times 32$\\
Conv1D 3x1 & $n \times n n_{mid}$\\
Reshape & $n \times n \times n_{mid}$ \\
$\Sigma\Pi$ & $n_{mid}$ \\
FC & 32 \\
FC & 32 \\
FC & 1 \\
\bottomrule
\end{tabular}
\end{sc}
\end{small}
\end{center}
\vspace{-0.1in}
\end{table}

\textbf{Maron} is an abbreviation of the $G$-invariant neural network architecture which is proved by \citet{maron19} to be a universal approximator.
In this case, one has to provide  $N_{inv}$ elements of the generating set of $G$-invariant polynomials,
whose degree is at most $|G|$ (by the Noether theorem \citep{kraft}), which was obtained by applying the Reynolds operator (see (\ref{eq:reynolds})) to all possible polynomials in $\mathbb{R}^{n \times n_in}$ with degree up to $|G|$ to the fully connected neural network.\\
It is worth to note that according to \citet{maron19}, a multiplication used to form the polynomials is approximated by a neural network, whose architecture is presented in Table \ref{tab:Maron_mul}.
A multi-layer perceptron (MLP) whose architecture is presented in Table \ref{tab:Maron_MLP}, was applied on these polynomials.

\begin{table*}[t]
\caption{Hyperparameters of the multiplication network used in the Maron architecture proposed by \citet{maron19}.}
\label{tab:Maron_mul}
\begin{center}
\begin{small}
\begin{sc}
\begin{tabular}{cc}
\toprule
Polynomial approximation & Area estimation \\
\midrule
FC layer: 64 output channels & FC layer: 32 output channels \\
FC layer: 32 output channels & FC layer: 1 output channel \\
FC layer: 1 output channel & \\
\bottomrule
\end{tabular}
\end{sc}
\end{small}
\end{center}
\vskip -0.15in
\end{table*}

\begin{table*}[htbp!]
\caption{Hyperparameters of the MLP network used in the Maron architecture proposed by \citet{maron19}.}
\label{tab:Maron_MLP}
\begin{center}
\begin{small}
\begin{sc}
\begin{tabular}{cc}
\toprule
Polynomial approximation & Area estimation \\
\midrule
FC layer: 48 output channels & FC layer: 40 output channels \\
FC layer: 192 output channels & FC layer: 1 output channel \\
FC layer: 32 output channels & \\
FC layer: 1 output channel & \\
\bottomrule
\end{tabular}
\end{sc}
\end{small}
\end{center}
\vskip -0.15in
\end{table*}

\begin{table*}[htbp!]
\caption{Exact formulas of the polynomials used in experiments given in Section \ref{sec:poly}, \ref{sec:group_size} and \ref{sec:robust} in the main paper.}
\label{tab:poly}
\begin{center}
\begin{small}
\begin{sc}
\begin{tabular}{ll}
\toprule
Invariance & Polynomial\\
\midrule
$\mathbb{Z}_5$ & $P_{\mathbb{Z}_5}(x) = x_1 x_2^2 + x_2 x_3^2 + x_3 x_4^2 + x_4 x_5^2 + x_5 x_1^2$\\
$\mathbb{Z}_3$ & $P_{\mathbb{Z}_3}(x) = x_1 x_2^2 + x_2 x_3^2 + x_3 x_1^2 + 2x_4 + x_5$\\
$S_3$ & $P_{S_3}(x) = x_1 x_2 x_3 + 2x_4 + x_5$\\
$S_3 \times S_2$ & $P_{S_3 \times S_2}(x) = x_1 x_2 x_3 + x_4 + x_5$\\
$D_8$ & $P_{D_8} = x_1 x_2^2 + x_2 x_3^2 + x_3 x_4^2 + x_4 x_1^2 + x_2 x_1^2 + x_3 x_2^2 + x_4 x_3^2 + x_1 x_4^2 + x_5$\\
$A_4$ & $P_{A_4} = x_1 x_2 + x_3 x_4 + x_1 x_3 + x_2 x_4 + x_1 x_4 + x_2 x_3 + x_1 x_2 x_3 + x_1 x_2 x_4 + x_1 x_3 x_4 + x_2 x_3 x_4 + x_5$\\
$S_4$ & $P_{S_4} = x_1 x_2 x_3 x_4 + x_5$\\
\bottomrule
\end{tabular}
\end{sc}
\end{small}
\end{center}
\vskip -0.15in
\end{table*}

\section{Datasets}
\subsection{Convex Quadrangle Area Estimation}
The dataset used in the task of convex quadrangle area estimation consists of a number of quadrangles with the associated area value.
Each of these quadrangles is defined by 8 numbers, while the associated area is the label for supervised learning.
Data generation procedure for the quadrangles consists of the following steps:
\begin{enumerate}
\setlength\itemsep{-0.1em}
    \item draw the value of the center of the quadrangle according to the uniform distribution,
    \item generate $n$ angles, in the range [0, $\frac{2\pi}{n}$],
    \item add $\frac{2k\pi}{n}$ to the $k$-th angle, for $k \in \{0, 1, \ldots, n-1\}$,
    \item draw uniformly the radius $r$,
    \item draw uniformly $n$ disturbances and add these values to the radius,
    \item generate the $x, y$ coordinates of vertices using generated angles and radii.
    \item take an absolute value of those coordinates (we want to have the coordinates positive),
    \item repeat steps 1--7 until obtained quadrangle is convex,
    \item calculate the area of the obtained quadrangle using the Monte Carlo method.
\end{enumerate}
Each of the training set and the validation set contains 256 examples, and 1024 examples were used in the test dataset.

\subsection{$G$-invariant Polynomial Approximation}
The dataset used in the tasks of $G$-invariant polynomial approximation consists of the input which is randomly generated and expected output,
which is simply calculated using formulas listed in Table \ref{tab:poly}. The generation procedure of the input draws samples from a uniform distribution between 0 and 1.
For the experiments given in Section \ref{sec:poly} and \ref{sec:quads}, the number of samples used for training, validation and test set is 16, 480 and 4800 respectively.
Only in the experiment given in Section \ref{sec:robust}, the number of samples in the training dataset was increased to 160, as the aim of the experiment was not analyzing the generalization properties,
but analyzing the ability to adjust the weights to capture other invariances.
To  train and test models using datasets with similar statistical properties, the seed for the data generation was set to 444.

\end{document}